\documentclass[sigplan,screen]{acmart}

\AtBeginDocument{%
  \providecommand\BibTeX{{%
    \normalfont B\kern-0.5em{\scshape i\kern-0.25em b}\kern-0.8em\TeX}}}

\usepackage[normalem]{ulem}
\usepackage{multirow}
\graphicspath{ {samples/figs/} }

\setcopyright{acmcopyright}
\copyrightyear{2021}
\acmYear{2021}
\acmDOI{10.1145/1122445.1122456}

\def\argmin{\mathop{\rm arg\,min}}

\newtheorem{theorem}{Theorem}
\newtheorem{remark}{Remark}
\newtheorem{definition}{Definition}

\newtheorem{assum}{Assumption}

\begin{document}

%%
%% The "title" command has an optional parameter,
%% allowing the author to define a "short title" to be used in page headers.
\title[Bidding via clustering ads intentions]{An Efficient Group-based Search Engine Marketing System for E-Commerce}

%%
%% The "author" command and its associated commands are used to define
%% the authors and their affiliations.
%% Of note is the shared affiliation of the first two authors, and the
%% "authornote" and "authornotemark" commands
%% used to denote shared contribution to the research.

\author{Cheng Jie, Da Xu, Zigeng Wang, Lu Wang, Wei Shen}
\affiliation{%
  \institution{Walmart Labs}
  \city{Sunnyvale}
  \state{California}
  \country{USA}
}
\email{{cheng.jie, da.xu, zigeng.wang0, lu.wang3, wei.shen}@walmart.com}

% \author{Cheng Jie}
% \affiliation{%
%   \institution{Walmart Labs}
%   \city{Sunnyvale}
%   \state{California}
%   \country{USA}
% }
% \email{Cheng.jie@walmart.com}

% \author{Da Xu}
% \affiliation{%
%   \institution{Walmart Labs}
%   \city{Sunnyvale}
%   \state{California}
%   \country{USA}
% }
% \email{Da.Xu@walmart.com}

% \author{Zigeng Wang}
% \affiliation{%
%   \institution{Walmart Labs}
%   \city{Sunnyvale}
%   \state{California}
%   \country{USA}
% }
% \email{zigeng.wang0@walmart.com}

% \author{Lu Wang}
% \affiliation{%
%   \institution{Walmart Labs}
%   \city{Sunnyvale}
%   \state{California}
%   \country{USA}
% }
% \email{lu.wang3@walmart.com}

%%
\renewcommand{\shortauthors}{Cheng Jie and Da Xu, et al.}

\begin{abstract}
With the increasing scale of the search-engine marketing, designing an efficient bidding system is becoming paramount important for the success of e-commerce company. The critical challenges faced by a modern industrial-level bidding system include: 1. the catalog is huge and the relevant bidding features are of high sparsity; 2. the large volume of bidding requests induces significant computation burden to both the offline and online serving. In this paper, we introduce the development and deployment process of the bidding system for search engine marketing on Walmart.com, which successfully handles hundreds of millions of biddings each day. We show and discuss the real-world performances of state-of-the-art deep learning methods, and reveal how we find their as the production-optimal solutions.
\end{abstract}


\maketitle

\section{Introduction}
In the modern era, online advertising has become a primary channel to deliver promotional marketing messages to customers. Among the various forms of online advertising, \emph{search engine marketing} (SEM) promotes business by showing and recommending advertisements on search-result pages. In particular, \emph{sponsored search auctions} contribute significantly to online advertising revenue as search results often have more prominent exposure.

The business impact of SEM attracts increasing attention from both academia and industry, specifically, the problems are appeal to the domains of economics and computer science\cite{Zhang2021Form1I}. 
Over the years, a large body of literature studies the constrained bidding optimization model, which maximizes business objectives under the prefixed spending limit. For instance, 
\cite{Feldman2007BudgetOI} and \cite{Borgs2007DynamicsOB} establish SEM bidding models for a single advertiser as constrained optimization problems in a deterministic setting where the advertisers' position, clicks, and the cost associated with a bid are known a priori. In comparison, SEM bidding as an optimization problem under the stochastic setting has been studied in \cite{Abhishek2012OptimalBI}. 
%Game-theoretic structures of SEM have been studied by \cite{Aggarwal2009GeneralAM} and \cite{Brgers2008EquilibriumBI}, and both works aims to boost the welfare of all advertisers in search engine platforms. 
More recently, there is a emerging stream of work which are dedicated to formulating the SEM bidding optimization as a reinforcement learning based \cite{la2016cumulative} dynamic pricing problem\cite{zhu2021news}, \cite{zhu2021time}, \cite{zhu2021clustering} by incorporating the sequential behavior of SEM ads \cite{Shen2020ReinforcementMD}.

While the seminal works have established rigorous mathematical properties for SEM bidding, their optimization models are often too restrictive for practical implementation. In real-world productions, the high volume of candidate ads is a crucial factor that hinders the applicability of those methods. The two challenges that stem from the industrial scale of ads are: (1). the feedback data of most SEM ads are inevitably sparse due to the limited slots at search engine platforms, preventing an accurate and effective estimation of their performances; (2). complex bidding evaluations are prohibitive because the volume of ads is enormous, and the frequency of bidding operations is very high. 

\subsection*{SEM bidding through clustering ads intentions}
In this paper, we introduce a generic bidding framework targeting the above challenges. The SEM solution proposed in our paper is currently in production for the multi-million-scale ads bidding for Walmart's e-commerce business. As we demonstrate in detail later in section \ref{section:system_overview}, the solution of our system comprises of two critical components:
\begin{itemize}
    \item a deep-learning-based multi-stage predictive algorithm for predicting the performance of the advertisement through their multi-modality signals, including the user feedback data and the contextual features of ads;
    \item an optimization algorithm that assigns a bidding price for each ad, based on its performance forecast for the desired business objectives.
\end{itemize}
Toward our goals, we first redesign a Transformer-based \cite{Vaswani2017AttentionIA} deep-learning language model to extract vector representations of the ads, which captures the customer's intentions when landing on the ads page. The advantage is that we can now fully leverage the geometric characteristics of the representations to aggregate ads' information that would be sparse otherwise (detailed in Section \ref{section:embedding_clustering}). The multi-stage prediction algorithm, which then augments the grouping patterns of features via ads clustering,  further alleviates the sparsity issue of the features. In the meantime, the clustering-based solution improves the scalability of the second-stage optimization algorithm by significantly reducing the number of entities in the downstream evaluation of the bids.

We thoroughly examine the performance of the proposed solutions via both offline studies and online experiments in Section \ref{section: experiment}. As expected, the clustering step is essential for trading off the sparsity, accuracy, and scalability. 

The previous literature addresses the sparsity issue primarily by using the ads' ``keywords'' in addition to the feedback data \cite{Hillard2010TheSO}. 
However, using word tokens as a categorical feature can pose severe problems in building predictive models due to the high cardinality. Unlike \cite{Hillard2010TheSO}, our approach constructs continuous vector representations of ads and therefore avoids the tenuous work of dealing with massive word tokens. We point out that the ideas of clustering SEM ads have also been proposed to overcome the high computation demands \cite{Chen2013QueryCB, zhu2020high}. However, the clustering algorithms developed in the above work are based on the distributions of SEM ads' historical feedback data, thereby excluding those with sparse historical features, which is problematic for modern SEM applications.

\section{Background for SEM Bidding}
\label{section:system_overview}

We first introduce the underlying bidding model and system that powers Walmart's SEM business.

\subsection{SEM bidding model with single objective}
Suppose an advertiser wants to maximize its revenue under the budget $\mathcal{B}$, we denote $\mathcal{K}$ as the set of ads which the advertiser would like to bid on. The SEM bidding model can be formulated as the following constrained optimization problem:
\begin{align}\label{eq:optimization}
& \max_{\{b_k\}}  E[ \sum_{k \in \mathcal {K}} R_k(b_k)]\\
& s.t.  \sum_{k \in \mathcal{K}} E[ C_k(b_k) ] \leq \mathcal{B}
\end{align}
where $R_k$, $C_k$ denote the revenue and cost function of ad $k$ with regard to the bid $b_k$.   

Practically, the cardinality of $\mathcal{K}$ is huge given the large number of items on list of the eCommerce merchandise's website. The large volume of ads raises two challenges: 1. It becomes impossible to accurately predict all the revenue function $R_k(b_k)$ and cost function $C_k(b_k)$. 2. Solving the KKT condition in \ref{eq: kkt} becomes intractable given the size of the problem.

In order to address the practical challenges for solving the large scale random optimization problem under constraint of \ref{eq:optimization}, we proposed the following approaches so that the magnitude of the optimization problem can be reduced dramatically.

First, inspired by the previous literature(\cite{Mohammad2009clustering}, \cite{Chen2013QueryCB}, \cite{LIN20181}) and the observation that many SEM ads indeed have very close customer intentions, we propose a bidding scheme where the bid is determined on the level of ad groups instead of individual ad. To this end, the bidding scheme will cluster the ads into groups based on their embedded vectors of deep-learning customer intention model, and set up the optimization model at the level of ad group. Through some notation changes, the SEM bidding model on ad group level can be formulated as follows:  
\begin{align}\label{eq:optimization-group}
& \max_{\{b_g\}}  E[ \sum_{g \in \mathcal {G}} R_g(b_g)]\\
& s.t.  \sum_{g \in \mathcal{G}} E\left[
C_g(b_g)\right]\leq \mathcal{B},
\end{align}
where $\mathcal {G}$ stands for the set of ad groups determined through the clustering step, and $b_g$ is the bidding value assigned to all the ads at group $g$. Based on our practical work of ad embedding and clustering, the ad group level bidding model usually only has $3\% - 5\%$ scale of the individual level bidding model. Section 4 will give a detailed discussion on the ads embedding and clustering steps.   

With the SEM bidding model set up on the level of ad group, finding an optimal solution to the problem is still quite onerous given that the forms of expected revenue $E(R_g(b_g))$ and cost functions $E(C_g(b_g))$ are unknown. However, by adding some practically reasonable assumptions on the expected revenue and cost functions, the optimal solution on \ref{eq:optimization-group} can be achieved in a quite efficient approach. Before presenting the assumption used for solving the problem \ref{eq:optimization-group}, we introduce the expected click function on bid value as $E[CLK(b)]$, and propose the following definitions of RPS(revenue per spend) and RPC(revenue per click) below.

\begin{definition}\label{roi-definition}
The RPS, i.e, revenue per spend, is defined as the amount of revenue an ad is bringing in on one unit of spend. Mathematically, for an ad group $g$, its $RPS_g = \frac{E[R_g(b_g)]}{E[C_g(b_g)]}$.
\end{definition}

\begin{definition}\label{roi-definition}
The $RPC$, i.e, revenue per click, is defined as the amount of revenue an ad is bringing in on one click, for a given ad $g$, we denote its revenue per click as $RPC_g$.
\end{definition}

\begin{assum}\label{eq:revenue_cost_assum}
We assume that for a given ad group $g$, its revenue per click $RPC_g$ is insensitive to the change of bid value $b_g$. Furthermore, we assume that expected click is linear function on $b_g$, i.e., $E[CLK_g(b_g)] = c_g b_g$.  
\end{assum}

\begin{remark}
With most of search engine platforms now switching to first-price auction model, the bid price $b_g$ now is essentially the cost of a click. As a result, based on assumption \ref{eq:revenue_cost_assum}, $E[C_g(b_g)] = c_g b^2_g$ for any ad group $g$. 
\end{remark}

Under the assumption \ref{eq:revenue_cost_assum} of expected revenue cost function, we can state and prove the following theorem regarding the optimal condition of the SEM optimization problem \ref{eq:optimization}. 

\begin{theorem}\label{thm:optimal-condition}
Under assumption \ref{eq:revenue_cost_assum}, the optimal solution of the optimization problem \ref{eq:optimization-group} is achieved when the RPS(revenue per investment) is equal for $\forall k$. 
\end{theorem}

\begin{proof}
The Lagrangian of \ref{eq:optimization-group} can be written as
\begin{align}
\mathcal{L} =  E[ \sum_{g \in \mathcal {G}} R_g(b_g)] - \lambda \{\mathcal{B} - \sum_{g \in \mathcal{G}} E\left[
C_g(b_g)\right] \},   
\end{align}
The KKT condition of \ref{eq:optimization-group} is
\begin{align} \label{eq: kkt}
    \forall g: \frac{d \mathcal{L}}{d b_g} = \frac{d}{d b_g} E[ \sum_{g \in \mathcal {G}} R_g(b_g)] - \lambda \frac{d}{d b_g} E[
C_g(b_g)] = 0, \lambda >= 0 \\
    \mathcal{B} - \sum_{g \in \mathcal{G}} E[
C_g(b_g)] >= 0
\end{align}
The KKT condition of \ref{eq: kkt} implies that an optimal solution exists when the following quotient 
\begin{align} \label{eq:sem-quotient}
\frac{d}{d b_g} E\left[ \sum_{g \in \mathcal {G}} R_g(b_g)\right]/\frac{d}{d b_g} E\left[
C_g(b_g)\right]
\end{align} 
takes the same value across $\forall g$. Under assumption \ref{eq:revenue_cost_assum}, we have 

\begin{align}
    E[R_g(b_g)] = E[CLK(b_g)] \cdot RPC_g = c_g b_g RPC_g \\
    E[C_g(b_g)] = E[CLK(b_g)] \cdot b_g = c_g b^2_g
\end{align}

\begin{align}\label{eq:kkt}
\frac{d}{d b_g} E\left[ \sum_{g \in \mathcal {G}} R_g(b_g)\right]/\frac{d}{d b_g} E\left[
C_g(b_g)\right] = \frac{RPC_g}{2b_g}
\end{align}
Notice that under assumption \ref{eq:revenue_cost_assum}, its revenue per spend $RPS_g = \frac{E(R_g(b_g))}{E(C_g(b_g))} = \frac{RPC_g}{b_g}$. Hence, the KKT condition \ref{eq:kkt} is equivalent to that $RPS_g$ are equal across $\forall g$.  
\end{proof}

Addition to the conclusion of theorem \ref{thm:optimal-condition}, recall that assumption \ref{eq:revenue_cost_assum} claims that $RPC_g$ are steady against $b_g$. Henceforth, the key step in optimizing \ref{eq:optimization-group} is predicting the revenue per click value for each ad group $g$. 

\begin{remark}
Note that the classical singular-ad bidding algorithm can be easily recovered by replacing the ad group $g$ with the single ad.     
\end{remark}

\begin{remark}
There exists scenarios that the business goal on SEM bidding system is pursuing a balanced optimal condition targeted at more than one objectives. In those cases, optimization function   \ref{eq:optimization-group} becomes a weighted sum of multiple expectations of different objectives. Similar assumptions and approaches as Assum \ref{eq:revenue_cost_assum} and thm \ref{thm:optimal-condition} can be applied to solve such optimization problem, which are skipped in this paper to reduce redundancy.   
\end{remark}

\begin{figure}[hbt]
    \centering
    \includegraphics[width=\linewidth]{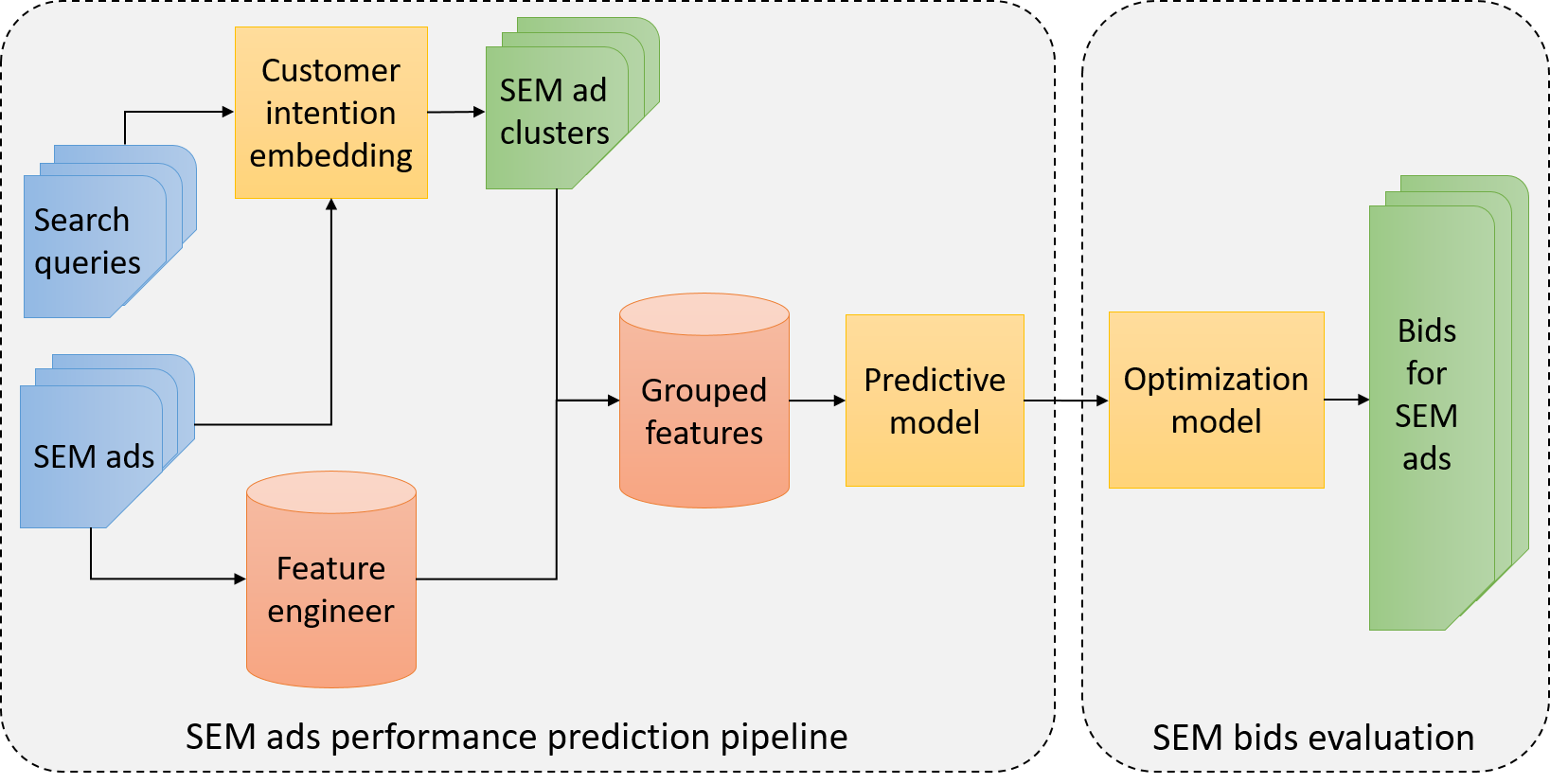}
    \caption{Overview of the infrastructure for providing SEM ads' bids by our approach.}
    \label{fig:sem_bidding_whole_sys}
\end{figure}

\subsection{SEM bidding system}
\label{sec:sem-bidding-system}
The results in the previous section suggest that the critical tasks for determining the bids of SEM ads is to accurately predict the revenue per click for each ad group $g$. In the sequel, we propose a design of the SEM ads bidding system illustrated in Fig \ref{fig:sem_bidding_whole_sys}. In Figure \ref{fig:sem_bidding_whole_sys}, the first task for obtaining the $RPC$ predictions is clustering the pool of SEM ads into ad groups. It consists of two steps: 1. building a representation learning model that encodes SEM ads into embeddings; 2. clustering SEM ads into ad groups.
After creating the SEM ad groups, the system will aggregate the features for ads within each ad group, and then train a predictive model to accurately forecast the $RPC_g$ for each ad group. We plug the $RPC_g$ back to the optimization problem and obtain the final bidding $b_g$ for each $g\in\mathcal{G}$ as $b_g = RPC_g/RPS_g$, where $RPS_g$ is known in advance. 

\section{Embedding and Clustering of SEM Ads}
\label{section:embedding_clustering}
The ad group level bidding model \ref{eq:optimization-group} requires that each ad cluster needs to represent an unique and exclusive customer purchase intention. Otherwise, if customer intentions of two ad groups overlap, the change of bid on one group might impact the impressions and conversions of another group, violating the bidding model assumption that each ad group behaves independently. To this end, we devise two steps in order to segment the SEM ads into mutually exclusive ad clusters in terms of customer intention: First, We build an customer intention representation model which can assign a vector representation of every ad. Second, based on the vector representations of each ad, we design a multi-stage clustering method in order to cluster large amount of ads into small to mid-sized ad groups.
  
\subsection{Customer intention embedding model}
% The customer intention of an SEM ad is defined as the integrated purchase intention (of the set of search queries) that leads to the clicked ads on the search engine. For example, an ad may appeal to customers who search for ``apple phone 8 case'' or ``iphone 9 case'', if their intentions are the case covers for various versions of iphone. If two ads share a large portion of clicked search queries, their customer intentions should be close to each other. Therefore, we design the customer intention model to reflect the co-click relations among the SEM ads. We propose the following metric to capture such intention.

In this section, we walk through the key component of the customer intention embedding model.
Customer intention of an SEM ad is defined as the integrated purchase intention of the set of search queries leading to clicks of the ad's web page on search engine. For example, an ad has drawn clicks of customers after searching queries like ``apple phone 8 case'' or ``iphone 9 case'' has customer intention be case covers for various versions of iphone. If two ads share a large portion of clicked search queries, their customer intentions should be close to each other by definition. Therefore, customer intention model is designed to reflect the co-click relationships among the SEM ads. 

\textbf{Interactive metric}. We propose a measure named as interactive metric(IM) in order to calibrate the extent of similarity between customer intentions of two SEM ads. Formally speaking, given two SEM ads A1 and A2, we first find the numbers of clicks of the two ads on their co-clicked queries and denote them as $CLK_{(A1coA2)}$ and $CLK_{(A2coA1)}$. With the numbers of total historical clicks of the two ads $CLK_{A1}$ and $CLK_{A2}$, the IM value between A1 and A2 is defined as 
\begin{align}\label{eq:interactive-metric}
IM_{A1, A2} = \sqrt{\frac{CLK_{(A1coA2)} * CLK_{(A2coA1)}}{CLK_{A1} * CLK_{A2}}},
\end{align}

Fig \ref{fig:interactive_factor} provide a real-world example to further demonstrate how interactive factor of two SEM ads is calculated. In fig \ref{fig:interactive_factor}, Ad1 and Ad2 have share search queries ``walmart chairs'' and ``Chairs Walmart'', along with corresponding 192 and 158 clicks on the shared queries. Moreover, the total number of historical clicks for Ad1 and 2 are 192 and 172 respectively, leading to the IM value between the two ads be $\sqrt{(192 * 158)/(192 * 172)} = 0.958$.

\textbf{Contextual features of SEM ad}.
When a search query appears, the search engine will try to match it with the SEM ads according to the content of their landing pages. In light of that, we select the ads' website's text content as the main feature for the customer intention model, since the content should be a critical factor in customers' decision making. The text feature of an SEM ad is a combination of titles and descriptions of products contained in the ad's website. For the SEM ads with more than one product, we choose the three top products to constrain the length of the input feature. Once the features are extracted, they are processed and converted through the standard tokenization and padding procedures described in \cite{Devlin2019BERTPO}.

\textbf{Transformer-based customer intention representation model}.
Recently, the attention-based encode-decode structure transformer has become the status quo architecture for natural language processing tasks \cite{Vaswani2017AttentionIA}. Motivated by the structure of the bidirectional transformer from \cite{Devlin2019BERTPO}, we built a transformer-based deep learning\cite{li2021frequentnet} model for extracting the customer intention from the text features of SEM ads. 
As we show in Fig \ref{fig:intention_representation}, for a given ad $A$ and its tokenized feature $T_A$, the model will consecutively go through an initial embedding layer, $3$ transformer layers, a dense pooling layer, and two feedforward layers before generating the final $512$-dimension \emph{normalized} output vector.

\textbf{Training data}.
The data we use for training the representation learning model is the $\texttt{search\_term\_report}$ from search engine, which provides the historical statistics of interactions (e.g. clicks, impressions) between SEM ads and their relevant search queries. Specifically, for each SEM ad, we will extract historical click numbers between the ad and each search query that leads to the clicks. 
Together with interactive metric $I$ defined at \ref{eq:interactive-metric}, we create a data-set $\mathcal{D}$ containing all the tuples of SEM ads having co-clicked queries together with their interactive metric. 
In addition to the above positive instances, we need negative instances to cover larger support of the distribution. For that purpose, we sample a certain number of ad tuples without co-clicked queries and append the tuples onto the data-set $\mathcal{D}$ by assigning them with an interactive metric value of  $-1$.
For the best practice, the ratio between positive tuples and negative tuples should approximately equal to the average positive interactive metric in the feedback data. 

\begin{figure}[h]
  \centering
  \includegraphics[width=\linewidth]{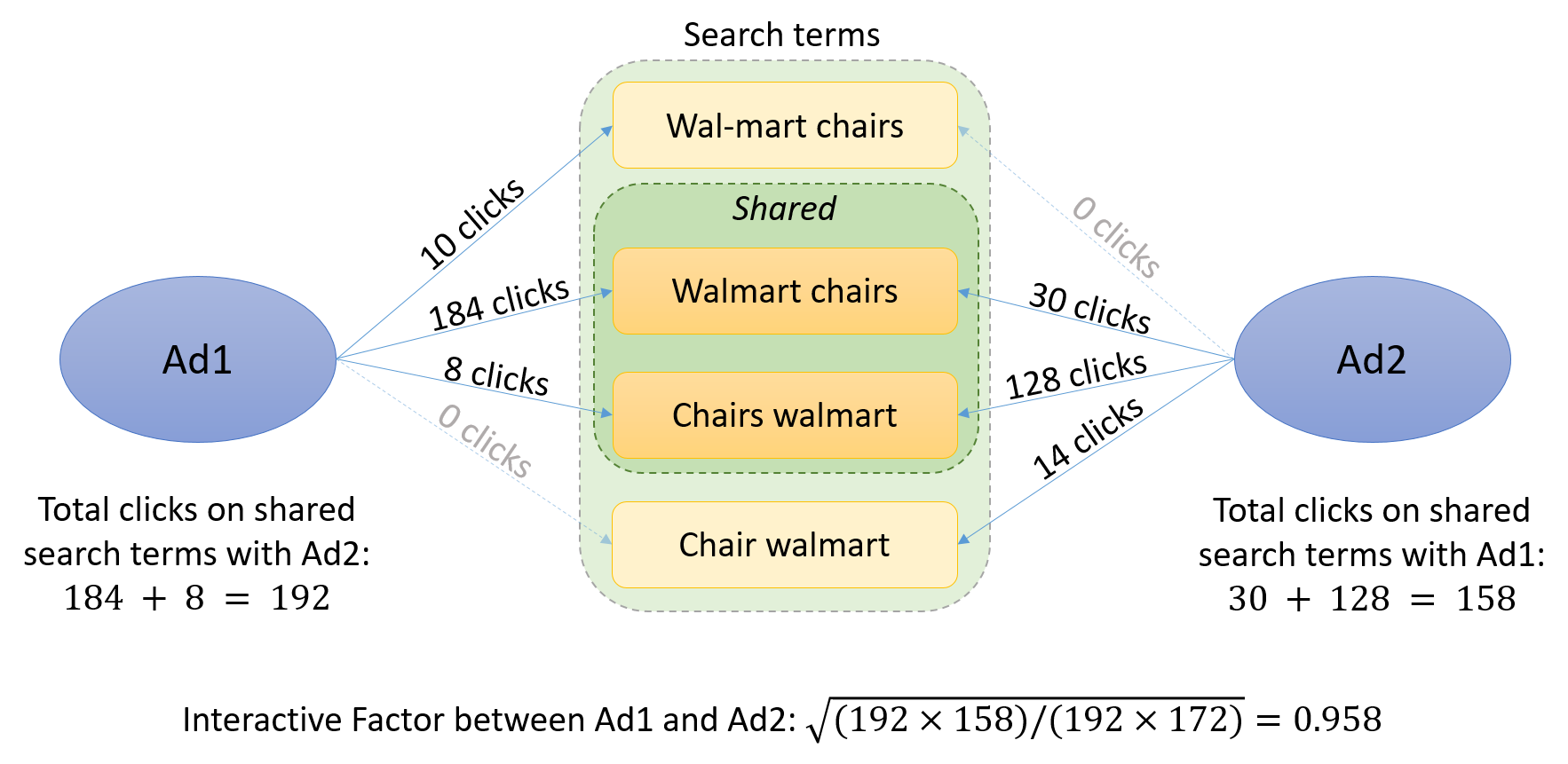}
  \caption{Interactive metric: an example}
   \label{fig:interactive_factor}
\end{figure}

\textbf{Model training}.
Let $f_\theta(\cdot)$ denote a customer intention model
with parameter vector $\theta$. Given an ad tuple $(A_i, A_j)$ along with their interactive metric $I_{ij}$, we define the loss function as
\begin{align}
-I_{ij}\log\sigma\big(f_\theta(T_{A_i})^Tf_\theta(T_{A_j})\big),
\end{align}
where $\sigma(\cdot)$ is the sigmoid function. The inner product of $f_\theta(T_{A_i})^Tf_\theta(T_{A_j})$ captures the cosine similarity between the embeddings of $(A_i, A_j)$, given that output vectors of the model $f_\theta(\cdot)$ are normalized. The structure of the model, together with the procedure for calculating the loss function, are presented in Figure \ref{fig:intention_representation}.
The optimization problem for finding the optimal $\theta $ is now given by: 
\begin{align}
\theta^{\star} = 
\argmin_{\theta \in \Theta}
\sum_{(A_i, A_j) \in DT} -I_{ij}\log\sigma(f_\theta(A_i)^Tf_\theta(A_j)), 
\label{eq:model-objective}
\end{align}
The objective (\ref{eq:model-objective}) indicates that the larger the interactive metric between two ads, the more impact this ad instance will carry when determining model parameter $\theta$. Including the negative instances will allow the model to further separate ads that lack a shared customer intention. Moreover, using negative samples can avoid over-fitting and the corner case where all SEM ads having a similar embedding. We use the ADAM\cite{Kingma2015AdamAM} optimizer, a variant of stochastic optimization \cite{Jie2018StochasticOI, Han2021ACP, Jie2018DecisionMU} for training (\ref{eq:model-objective}).

\begin{figure}[h]
  \centering
  \includegraphics[width=\linewidth]{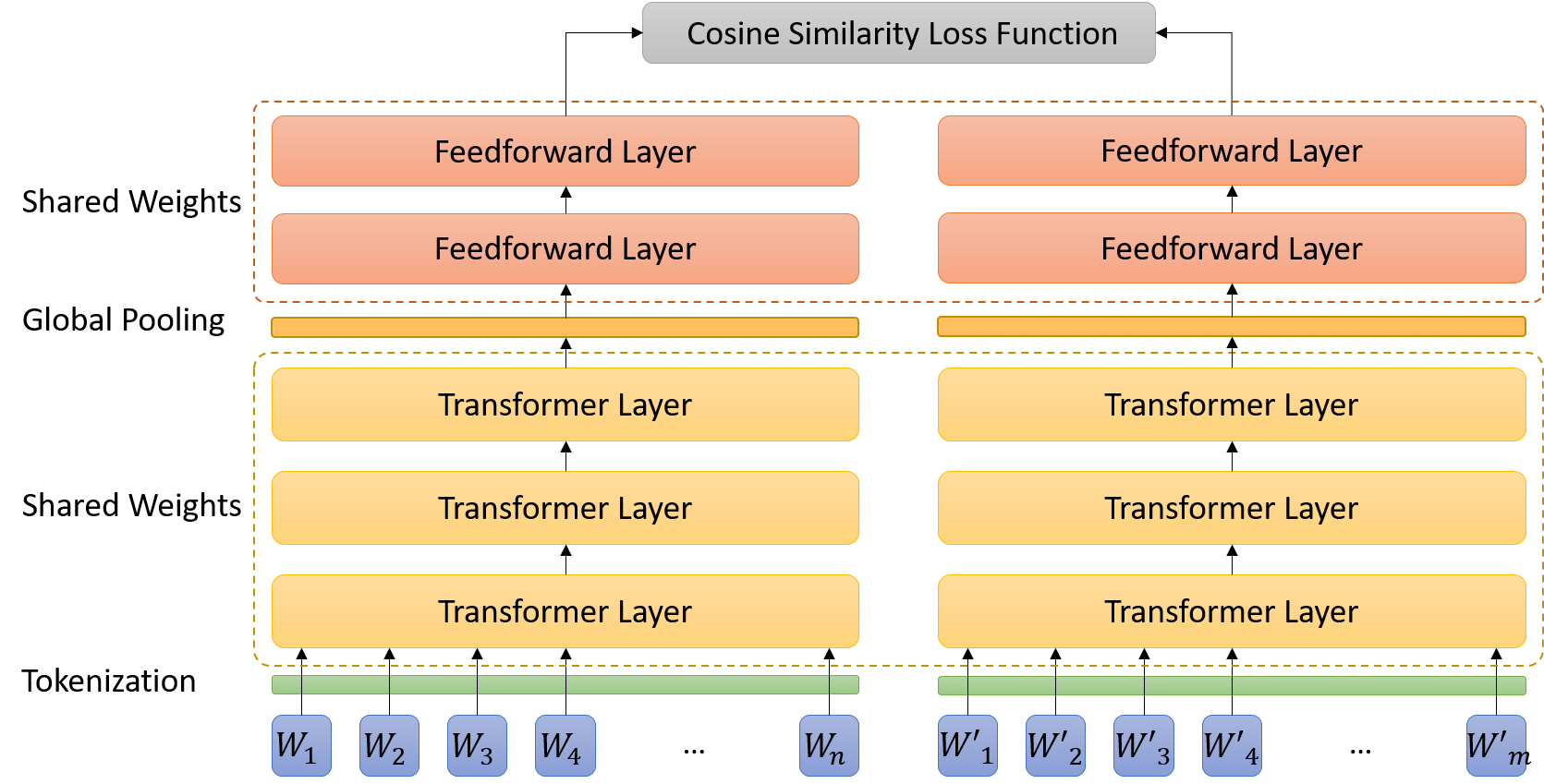}
  \caption{SEM ads customer intention embedding model}
   \label{fig:intention_representation}
\end{figure}

\subsection{Multi-stage SEM ads clustering algorithm}
In what follows, we discuss clustering with ads embedding. Due to the high volume of ads in modern SEM, it is impractical to apply the clustering algorithms that require computing all the pair-wise distances. Here, we present a multi-stage method that leverages the SEM ads' taxonomy and significantly reduces the computation demand.

\textbf{SEM ads classification}. The first step of the multi-stage clustering algorithm is to classify each SEM ad into one of the \emph{product types}, which can be any taxonomy that is labeled for the items: electronics, beverage, etc. Most companies have the predefined taxonomy for each item, which should be actively exploited. SEM ad with only one item can be directly concluded to its product type, and serve as the training sample of the taxonomy classification model. For SEM ads with more than one item, we train a feedforward neural network\cite{Du2020MultipleSK} to predict each ads' product type, which takes the embedding of the SEM ad as input. 
 
\textbf{Clustering within each product type}.
Following the classification, we apply the "bottom-up" Agglomerative clustering using embedding vectors as features to create mutually exclusive ad groups for the SEM ads within each product type. Naturally, the cosine distance is employed as the \emph{linkage metric}, and it also allows us to determine the threshold based on which the final clusters are formed. We point out that the first classification step significantly reduces the computation complexity compared with directly clustering all the ads. 

\section{Predicting RPC for SEM ads}
\label{section:performance_prediction}
In the next step, we build a machine learning model for each ads cluster to predict the key quantity of $RPC_g$, i.e. the revenue per click, whose role was illustrated in Section \ref{sec:sem-bidding-system}.

\textbf{Features}. The features we use for predicting $RPC$ can be categorized into three classes: 1. the historical feedback statistics such as clicks and conversions; 2. the activity metric for the ad's landing pages such as bounce rate; 3. contextual features of the ad, such as aggregating the ads embedding for each ad group. 

\textbf{Model selection}.
There are varieties of machine learning models\cite{jarrow2021low, zhu2020adaptive} for predicting $RPC$. The most basic prediction model is linear regression, and its apparent advantages are interpret-ability and computational efficiency. However, linear regression model usually suffers from low performance especially on the SEM ad groups features with complex structures. Besides linear regression model, we also tried more complex alternatives such as boosted regression tree\cite{Friedman2001GreedyFA, ZHAO2018619} and neural network. It turned out that when applied on our tabular of features and responses($RPC$), neural network usually doesn't attain as high accuracy as boosted regression tree, according to our cross-validation performance. Moreover, with recent developments of highly efficient gradient boosting tree package such as \emph{XGboost} and \emph{LightGBM}, training an accurate gradient boosting tree algorithm is manageable on the scale of our SEM ads' features data. Tree explainers such as feature importance and SHAP value also assist us in understanding the impact of each feature in predicting $RPC$, although not as straightforward as linear regression model. As a result, taking all factors(accuracy, interpret-ability, robustness etc) altogether into consideration, we finally select \emph{Gradient Boosting Tree} as our $RPC$ prediction model $r$. In the next section, we will compare the prediction performances among linear regression, gradient boosting tree and neural network.

\textbf{Model training}
We choose the clicks-weighted square error as the loss function for model training because ad groups with higher clicks often have more impact on the business. 
Formally, by denoting the parameter of the model by $\eta\in\mathcal{H}$ and the total clicks of ad group by $C_g$, the objective function is given by: 
\begin{align}
\eta^\star = 
\argmin_{\eta \in \mathcal{H}}
\frac{\sum_{g \in \mathcal{G}} C_g(r_\eta(X_g) - RPC_g)^2}{\sum_{g \in \mathcal{G}}C_g},
\end{align}
where $r_{\eta}$ is the RPC predicitive model.

\section{Experiments and Analysis}
\label{section: experiment}
We conducted both offline and online experiments to answer the following questions:

\textbf{Q1:} Can ads clustering improve the RPC prediction accuracy by addressing the spareness of feedback data? 

\textbf{Q2:} Does the proposed two-step framework improve the business performance?

\subsection*{Offline experiment: prediction accuracy comparison}
The offline experiment is designed to test whether the proposed clustering methods address the sparseness issue and improve RPC prediction accuracy. To this end, we select a set of ads with a total number of  $\sim$20 million, and compare the accuracy of RPC predictions of 1. directly applying RPC prediction on each SEM ads (the baseline singular-ad-based algorithm); 2. clustering SEM ads before predicting RPC for each ad cluster (our cluster-based bidding algorithm). For fair comparison, we evaluate the performance metric based on each ad and set the predicted RPC of each ad equivalent to the predicted RPC of its belonging ad cluster when using the second approach. 
According to the operation protocol of Walmart, we predict the weekly RPC as described in Section \ref{section:performance_prediction}. For the proposed approach, we apply the methods introduced in Section \ref{section:embedding_clustering} to cluster SEM ads into ad groups, and aggregate the ad features within each ad group. The summary statistics for the ad groups and the original SEM ads are displayed in Table \ref{table:sem-ad-adgroups}. 

\begin{table}[hbt]
\centering
\caption{SEM ads vs ad groups: Data-set overview}
\resizebox{\linewidth}{!}{
\begin{tabular}{c|cc}
    \toprule
        & SEM ads & SEM ad groups \\ 
    \midrule
    Dataset sample size & 19.6 M & 1.8M \\ 
    Missing feature (proportion) & 91.6$\%$ & 54.4$\%$ \\
    Non-empty response ratio & 6.7 $\%$ & 36.4 $\%$ \\
    Relative response RPC variance & 100$\%$ & 55$\%$ \\
    \bottomrule
  \end{tabular}
  }
\label{table:sem-ad-adgroups}
\end{table}

In table \ref{table:sem-ad-adgroups}, the proportion of feature missingness is calculated based on the non-contextual features, and due to Walmart's privacy policy, the variances of the RPC response variable are presented as percentage proportions to the largest among the two datasets. Table \ref{table:sem-ad-adgroups} manifest the two benefits of ads clustering: 1. the feature sparseness is dramatically improved as exemplified by the reduced missing feature proportion, 2. the reduced variance of the response variable indicates that the clustering algorithm tends to produce a more robust output for the downstream $RPG$ modeling.

We experimented with three machine learning models for predicting the weekly RPC: \emph{linear regression} (LR) model, \emph{TabNet} and \emph{gradient boosting regression tree} (GBRT). We split the dataset into training, validation, and test by $80\%-10\%-10\%$, where the test dataset is used to report the predictive accuracy of the trained models. In addition to the click-weighted MSE (WMSE) mentioned at section \ref{section:performance_prediction}, we also include the click-weighted MAE (WMAE) as performance metric. The performances of the trained models are displayed on table \ref{table:prediction-accuracy-time}. Figure \ref{fig:relative_wmse} presents an example of the gradient boosting trees when applied to the baseline and our approach, under their best hyper-parameter combinations.
Due to the privacy policy, we provide the accuracy metric with respect to the baseline model, which is Linear regression(LR) on the singular-ad-based algorithm. The model training, including hyper-parameter tuning is conducted on a Linux system with 64 core 2.80GHz CPUs and 800 GB memory.  

\begin{figure}[h]
  \centering
  \includegraphics[width=0.7\linewidth]{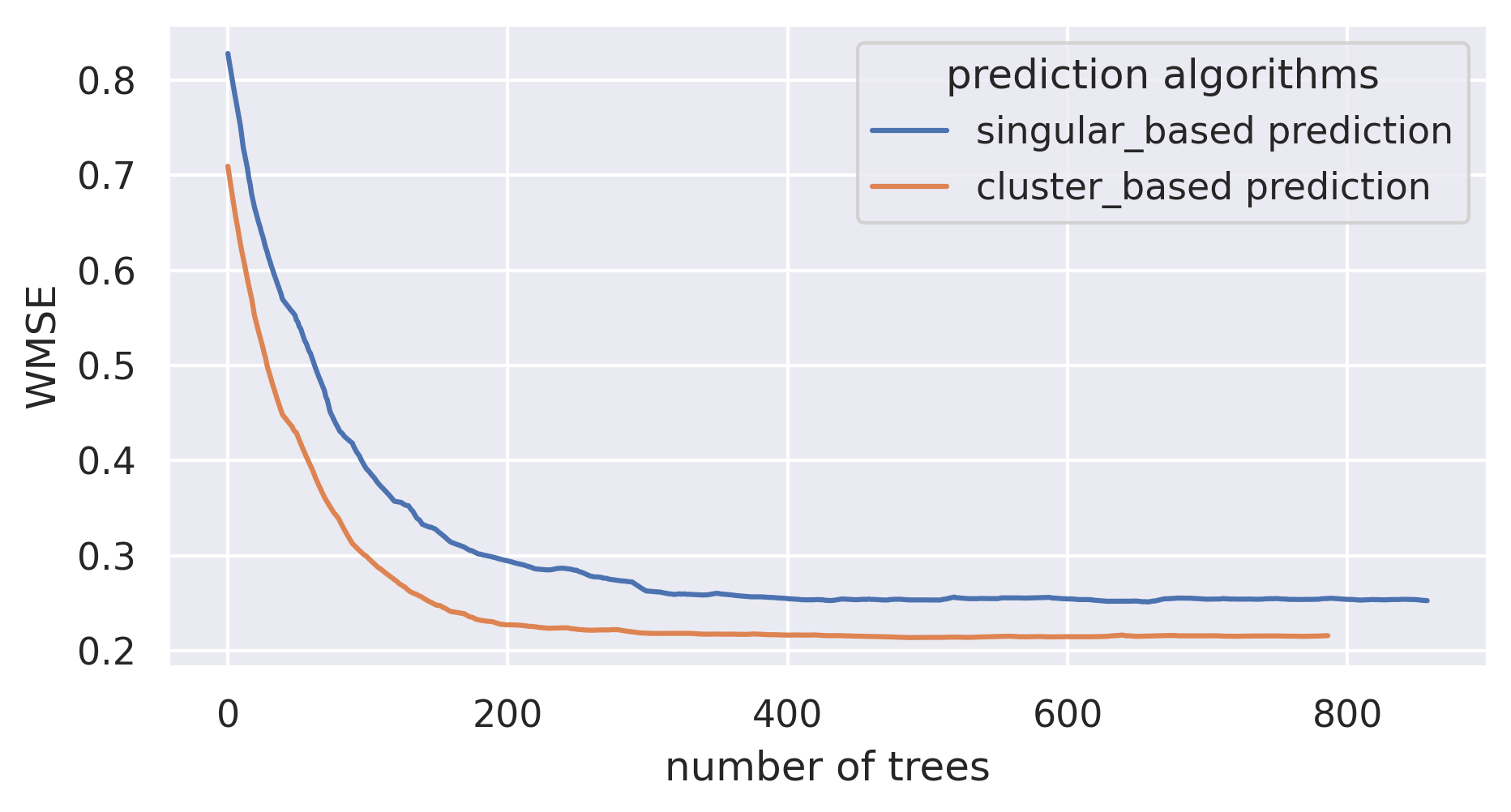}
  \caption{Relative WMSE of the baseline and our approach when using GBRT.}
  % \Description{A woman and a girl in white dresses sit in an open car.}
   \label{fig:relative_wmse}
\end{figure}

\begin{table*}
  \centering
  \caption{The RPC prediction accuracy (relative to LR on singular-ad setting), and the offline model training time.}
    \begin{tabular}{c|cc|cc|cc}
    \toprule
    \multirow{2}[2]{*}{Predictive model} & \multicolumn{2}{c|}{WMSE(Relative to LR Singular)} &
    \multicolumn{2}{c|}{WMAE(Relative to LR Singular)} &
    \multicolumn{2}{c}{Training time} \\
          & Singular based & \multicolumn{1}{c|}{Cluster based} & Singular based & \multicolumn{1}{c|}{Cluster based} & Singular based & Cluster based \\
    \midrule
    LR(reference point) & 100$\%$  & 92$\%$  & 100$\%$ & 86$\%$ & 8m & 2m \\
    TabNet & 30$\%$ & 27$\%$  & 35$\%$ & 31$\%$ & 15h  & 2.5h \\
    Gradient boosting & 25$\%$ & 21$\%$  & 29$\%$ & 23$\%$ & 16h & 2h \\
    \bottomrule
    \end{tabular}
  \label{table:prediction-accuracy-time}
\end{table*}

The results from Table \ref{table:prediction-accuracy-time} and Figure \ref{fig:relative_wmse} suggest that RPC prediction via ad clustering consistently achieves better performances compared with the singular ad prediction. Further, the computational time for training RPC at the cluster level is considerably less than the singular-ad level. 
\subsection*{Online experiment: business efficiency comparison} 
We design the online AB testing experiment to see whether the clustering-based bidding improves business performance, which is mainly reflected by the revenue per spend (RPS). We use the stratified sampling to select 200,000 SEM ads across different product types as our target ads pool, and compare the RPS of clustering-based bidding algorithm and traditional singular-ad-based bidding algorithm when applied on the selected SEM ads. 

Here, we leverage the Draft $\&$ Experiment platform from Google Adwords\footnote{https://ads.google.com, where the max capacity for a campaign is 200,000.} to create a pair of control and test campaigns that host the 200k SEM ads. The singular-based and clustering-based bidding algorithms are applied to the control and test campaign, respectively. We set a common RPS goal for the two algorithms to evaluate bids. 

\begin{table}
\caption{Online AB testing results.}
\begin{tabular}{cc|cc}
    \toprule
    Test  & Metric(Relative to control) & Control & Test \\ 
    \midrule
    AA period& Spend & 100$\%$ & 101$\%$ \\ 
     & RPS & 100$\%$ & 99$\%$ \\
    AB period& Spend & 100$\%$ & 102$\%$ \\
     & RPS & 100$\%$ & 109$\%$ \\
     \bottomrule
  \end{tabular}
\label{table:online-experiment-results}
\end{table}
%The control and test campaigns are launched simultaneously, 
%and during the test, Google evenly split incoming traffic to ensure a fair comparison. 
The experiment session consists of one week of AA test, and one following week of AB test. During the \textit{AB} period, we keep the spending between control and test campaigns close through some proportional bids adjustments. 

The online testing results are presented in table \ref{table:online-experiment-results}, where we present metrics relative to the control campaign. The test campaign's RPS on the AB period exemplifies that clustering-based bidding is able to achieve improved performance compared with the singular-ad-based-bidding. 
%To further justify our conclusion, we perform a paired \emph{t test} on the RPS of two campaigns, which shows that our experiment reached a t-statistics of 2.1 with the p-value of 0.02.   
\section{Conclusion}
This paper introduces a two-step clustering-based SEM bidding system that integrates modern representation learning with the Transformer language model. We describe the detailed development infrastructure that may bring insights to both practitioners and researchers in this domain. The offline and online experiments show that the proposed system compares favorably to the alternatives in terms of accuracy and training efficiency. 
Our successful deployment for Walmart e-commerce further reveals combining clustering with modern representation learning as a scalable solution for industrial bidding systems. 

\bibliographystyle{ACM-Reference-Format}
\bibliography{irs_workshop_arxiv}

%%% -*-BibTeX-*-
%%% Do NOT edit. File created by BibTeX with style
%%% ACM-Reference-Format-Journals [18-Jan-2012].

\begin{thebibliography}{26}

%%% ====================================================================
%%% NOTE TO THE USER: you can override these defaults by providing
%%% customized versions of any of these macros before the \bibliography
%%% command.  Each of them MUST provide its own final punctuation,
%%% except for \shownote{}, \showDOI{}, and \showURL{}.  The latter two
%%% do not use final punctuation, in order to avoid confusing it with
%%% the Web address.
%%%
%%% To suppress output of a particular field, define its macro to expand
%%% to an empty string, or better, \unskip, like this:
%%%
%%% \newcommand{\showDOI}[1]{\unskip}   % LaTeX syntax
%%%
%%% \def \showDOI #1{\unskip}           % plain TeX syntax
%%%
%%% ====================================================================

\ifx \showCODEN    \undefined \def \showCODEN     #1{\unskip}     \fi
\ifx \showDOI      \undefined \def \showDOI       #1{#1}\fi
\ifx \showISBNx    \undefined \def \showISBNx     #1{\unskip}     \fi
\ifx \showISBNxiii \undefined \def \showISBNxiii  #1{\unskip}     \fi
\ifx \showISSN     \undefined \def \showISSN      #1{\unskip}     \fi
\ifx \showLCCN     \undefined \def \showLCCN      #1{\unskip}     \fi
\ifx \shownote     \undefined \def \shownote      #1{#1}          \fi
\ifx \showarticletitle \undefined \def \showarticletitle #1{#1}   \fi
\ifx \showURL      \undefined \def \showURL       {\relax}        \fi
% The following commands are used for tagged output and should be
% invisible to TeX
\providecommand\bibfield[2]{#2}
\providecommand\bibinfo[2]{#2}
\providecommand\natexlab[1]{#1}
\providecommand\showeprint[2][]{arXiv:#2}

\bibitem[\protect\citeauthoryear{Abhishek and Hosanagar}{Abhishek and
  Hosanagar}{2012}]%
        {Abhishek2012OptimalBI}
\bibfield{author}{\bibinfo{person}{V. Abhishek} {and} \bibinfo{person}{K.
  Hosanagar}.} \bibinfo{year}{2012}\natexlab{}.
\newblock \showarticletitle{Optimal bidding in multi-item multi-slot sponsored
  search auctions}. In \bibinfo{booktitle}{\emph{EC '12}}.
\newblock


\bibitem[\protect\citeauthoryear{Borgs, Chayes, Immorlica, Jain, Etesami, and
  Mahdian}{Borgs et~al\mbox{.}}{2007}]%
        {Borgs2007DynamicsOB}
\bibfield{author}{\bibinfo{person}{C. Borgs}, \bibinfo{person}{J. Chayes},
  \bibinfo{person}{Nicole Immorlica}, \bibinfo{person}{K. Jain},
  \bibinfo{person}{O. Etesami}, {and} \bibinfo{person}{Mohammad Mahdian}.}
  \bibinfo{year}{2007}\natexlab{}.
\newblock \showarticletitle{Dynamics of bid optimization in online
  advertisement auctions}. In \bibinfo{booktitle}{\emph{WWW '07}}.
\newblock


\bibitem[\protect\citeauthoryear{Chen, Liu, Yi, Schwaighofer, and Yan}{Chen
  et~al\mbox{.}}{2013}]%
        {Chen2013QueryCB}
\bibfield{author}{\bibinfo{person}{Y. Chen}, \bibinfo{person}{W. Liu},
  \bibinfo{person}{J. Yi}, \bibinfo{person}{Anton Schwaighofer}, {and}
  \bibinfo{person}{T.~W. Yan}.} \bibinfo{year}{2013}\natexlab{}.
\newblock \showarticletitle{Query clustering based on bid landscape for
  sponsored search auction optimization}.
\newblock \bibinfo{journal}{\emph{Proceedings of the 19th ACM SIGKDD
  international conference on Knowledge discovery and data mining}}
  (\bibinfo{year}{2013}).
\newblock


\bibitem[\protect\citeauthoryear{Devlin, Chang, Lee, and Toutanova}{Devlin
  et~al\mbox{.}}{2019}]%
        {Devlin2019BERTPO}
\bibfield{author}{\bibinfo{person}{J. Devlin}, \bibinfo{person}{Ming-Wei
  Chang}, \bibinfo{person}{Kenton Lee}, {and} \bibinfo{person}{Kristina
  Toutanova}.} \bibinfo{year}{2019}\natexlab{}.
\newblock \showarticletitle{BERT: Pre-training of Deep Bidirectional
  Transformers for Language Understanding}. In
  \bibinfo{booktitle}{\emph{NAACL-HLT}}.
\newblock


\bibitem[\protect\citeauthoryear{Du, Zhang, Shi, and Chen}{Du
  et~al\mbox{.}}{2020}]%
        {Du2020MultipleSK}
\bibfield{author}{\bibinfo{person}{Tianming Du}, \bibinfo{person}{Yanci Zhang},
  \bibinfo{person}{Xiaotong Shi}, {and} \bibinfo{person}{Shuang Chen}.}
  \bibinfo{year}{2020}\natexlab{}.
\newblock \showarticletitle{Multiple Slice k-space Deep Learning for Magnetic
  Resonance Imaging Reconstruction}.
\newblock \bibinfo{journal}{\emph{2020 42nd Annual International Conference of
  the IEEE Engineering in Medicine \& Biology Society (EMBC)}}
  (\bibinfo{year}{2020}), \bibinfo{pages}{1564--1567}.
\newblock


\bibitem[\protect\citeauthoryear{Feldman, Muthukrishnan, P{\'a}l, and
  Stein}{Feldman et~al\mbox{.}}{2007}]%
        {Feldman2007BudgetOI}
\bibfield{author}{\bibinfo{person}{J. Feldman}, \bibinfo{person}{S.
  Muthukrishnan}, \bibinfo{person}{Martin P{\'a}l}, {and} \bibinfo{person}{C.
  Stein}.} \bibinfo{year}{2007}\natexlab{}.
\newblock \showarticletitle{Budget optimization in search-based advertising
  auctions}.
\newblock \bibinfo{journal}{\emph{ArXiv}}  \bibinfo{volume}{abs/cs/0612052}
  (\bibinfo{year}{2007}).
\newblock


\bibitem[\protect\citeauthoryear{Friedman}{Friedman}{2001}]%
        {Friedman2001GreedyFA}
\bibfield{author}{\bibinfo{person}{J. Friedman}.}
  \bibinfo{year}{2001}\natexlab{}.
\newblock \showarticletitle{Greedy function approximation: A gradient boosting
  machine.}
\newblock \bibinfo{journal}{\emph{Annals of Statistics}}  \bibinfo{volume}{29}
  (\bibinfo{year}{2001}), \bibinfo{pages}{1189--1232}.
\newblock


\bibitem[\protect\citeauthoryear{Han, Mazouchi, Nageshrao, and Modares}{Han
  et~al\mbox{.}}{2021}]%
        {Han2021ACP}
\bibfield{author}{\bibinfo{person}{Yuzhen Han}, \bibinfo{person}{Majid
  Mazouchi}, \bibinfo{person}{S. Nageshrao}, {and} \bibinfo{person}{H.
  Modares}.} \bibinfo{year}{2021}\natexlab{}.
\newblock \showarticletitle{A Convex Programming Approach to Data-Driven
  Risk-Averse Reinforcement Learning}.
\newblock \bibinfo{journal}{\emph{ArXiv}}  \bibinfo{volume}{abs/2103.14606}
  (\bibinfo{year}{2021}).
\newblock


\bibitem[\protect\citeauthoryear{Hillard, Manavoglu, Raghavan, Leggetter,
  Cant{\'u}-Paz, and Iyer}{Hillard et~al\mbox{.}}{2010}]%
        {Hillard2010TheSO}
\bibfield{author}{\bibinfo{person}{D. Hillard}, \bibinfo{person}{Eren
  Manavoglu}, \bibinfo{person}{H. Raghavan}, \bibinfo{person}{C. Leggetter},
  \bibinfo{person}{E. Cant{\'u}-Paz}, {and} \bibinfo{person}{R. Iyer}.}
  \bibinfo{year}{2010}\natexlab{}.
\newblock \showarticletitle{The sum of its parts: reducing sparsity in click
  estimation with query segments}.
\newblock \bibinfo{journal}{\emph{Information Retrieval}}  \bibinfo{volume}{14}
  (\bibinfo{year}{2010}), \bibinfo{pages}{315--336}.
\newblock


\bibitem[\protect\citeauthoryear{Jarrow, Murataj, Wells, and Zhu}{Jarrow
  et~al\mbox{.}}{2021}]%
        {jarrow2021low}
\bibfield{author}{\bibinfo{person}{Robert~A Jarrow}, \bibinfo{person}{Rinald
  Murataj}, \bibinfo{person}{Martin~T Wells}, {and} \bibinfo{person}{Liao
  Zhu}.} \bibinfo{year}{2021}\natexlab{}.
\newblock \showarticletitle{The Low-volatility Anomaly and the Adaptive
  Multi-Factor Model}.
\newblock \bibinfo{journal}{\emph{arXiv preprint arXiv:2003.08302}}
  (\bibinfo{year}{2021}).
\newblock


\bibitem[\protect\citeauthoryear{Jie}{Jie}{2018}]%
        {Jie2018DecisionMU}
\bibfield{author}{\bibinfo{person}{Cheng Jie}.}
  \bibinfo{year}{2018}\natexlab{}.
\newblock \showarticletitle{Decision Making Under Uncertainty: New Models and
  Applications}.
\newblock


\bibitem[\protect\citeauthoryear{Jie, PrashanthL., Fu, Marcus, and
  Szepesvari}{Jie et~al\mbox{.}}{2018}]%
        {Jie2018StochasticOI}
\bibfield{author}{\bibinfo{person}{Cheng Jie}, \bibinfo{person}{A.
  PrashanthL.}, \bibinfo{person}{M. Fu}, \bibinfo{person}{S. Marcus}, {and}
  \bibinfo{person}{Csaba Szepesvari}.} \bibinfo{year}{2018}\natexlab{}.
\newblock \showarticletitle{Stochastic Optimization in a Cumulative Prospect
  Theory Framework}.
\newblock \bibinfo{journal}{\emph{IEEE Trans. Automat. Control}}
  \bibinfo{volume}{63} (\bibinfo{year}{2018}), \bibinfo{pages}{2867--2882}.
\newblock


\bibitem[\protect\citeauthoryear{Kingma and Ba}{Kingma and Ba}{2015}]%
        {Kingma2015AdamAM}
\bibfield{author}{\bibinfo{person}{Diederik~P. Kingma} {and}
  \bibinfo{person}{Jimmy Ba}.} \bibinfo{year}{2015}\natexlab{}.
\newblock \showarticletitle{Adam: A Method for Stochastic Optimization}.
\newblock \bibinfo{journal}{\emph{CoRR}}  \bibinfo{volume}{abs/1412.6980}
  (\bibinfo{year}{2015}).
\newblock


\bibitem[\protect\citeauthoryear{Li, Song, Sun, and Zhu}{Li
  et~al\mbox{.}}{2021}]%
        {li2021frequentnet}
\bibfield{author}{\bibinfo{person}{Yifei Li}, \bibinfo{person}{Kuangyan Song},
  \bibinfo{person}{Yiming Sun}, {and} \bibinfo{person}{Liao Zhu}.}
  \bibinfo{year}{2021}\natexlab{}.
\newblock \showarticletitle{FrequentNet: A Novel Interpretable Deep Learning
  Model for Image Classification}.
\newblock \bibinfo{journal}{\emph{Available at SSRN:
  https://ssrn.com/abstract=3895462}} (\bibinfo{year}{2021}).
\newblock


\bibitem[\protect\citeauthoryear{Lin, Jie, and Marcus}{Lin
  et~al\mbox{.}}{2018}]%
        {LIN20181}
\bibfield{author}{\bibinfo{person}{Kun Lin}, \bibinfo{person}{Cheng Jie}, {and}
  \bibinfo{person}{Steven~I. Marcus}.} \bibinfo{year}{2018}\natexlab{}.
\newblock \showarticletitle{Probabilistically distorted risk-sensitive
  infinite-horizon dynamic programming}.
\newblock \bibinfo{journal}{\emph{Automatica}}  \bibinfo{volume}{97}
  (\bibinfo{year}{2018}), \bibinfo{pages}{1--6}.
\newblock
\showISSN{0005-1098}
\urldef\tempurl%
\url{https://doi.org/10.1016/j.automatica.2018.07.028}
\showDOI{\tempurl}


\bibitem[\protect\citeauthoryear{Mohammad~Mahdian}{Mohammad~Mahdian}{2009}]%
        {Mohammad2009clustering}
\bibfield{author}{\bibinfo{person}{Grant~Wang Mohammad~Mahdian}.}
  \bibinfo{year}{2009}\natexlab{}.
\newblock \showarticletitle{Clustering-Based Bidding languages for Sponsored
  Search}.
\newblock \bibinfo{journal}{\emph{European Symposium on Algorithms}}
  (\bibinfo{year}{2009}).
\newblock


\bibitem[\protect\citeauthoryear{Prashanth, Jie, Fu, Marcus, and
  Szepesv{\'a}ri}{Prashanth et~al\mbox{.}}{2016}]%
        {la2016cumulative}
\bibfield{author}{\bibinfo{person}{L.A. Prashanth}, \bibinfo{person}{Cheng
  Jie}, \bibinfo{person}{Michael Fu}, \bibinfo{person}{Steve Marcus}, {and}
  \bibinfo{person}{Csaba Szepesv{\'a}ri}.} \bibinfo{year}{2016}\natexlab{}.
\newblock \showarticletitle{Cumulative Prospect Theory Meets Reinforcement
  Learning: Prediction and Control}. In \bibinfo{booktitle}{\emph{Proceedings
  of The 33rd International Conference on Machine Learning}}.
  \bibinfo{pages}{1406--1415}.
\newblock


\bibitem[\protect\citeauthoryear{Shen, Peng, Liu, Zhang, Qian, Hong, Guo, Ding,
  Lu, and Tang}{Shen et~al\mbox{.}}{2020}]%
        {Shen2020ReinforcementMD}
\bibfield{author}{\bibinfo{person}{W. Shen}, \bibinfo{person}{Binghui Peng},
  \bibinfo{person}{Hanpeng Liu}, \bibinfo{person}{Michael Zhang},
  \bibinfo{person}{Ruohan Qian}, \bibinfo{person}{Y. Hong}, \bibinfo{person}{Z.
  Guo}, \bibinfo{person}{Zongyao Ding}, \bibinfo{person}{Pengjun Lu}, {and}
  \bibinfo{person}{Pingzhong Tang}.} \bibinfo{year}{2020}\natexlab{}.
\newblock \showarticletitle{Reinforcement Mechanism Design, with Applications
  to Dynamic Pricing in Sponsored Search Auctions}.
\newblock \bibinfo{journal}{\emph{ArXiv}}  \bibinfo{volume}{abs/1711.10279}
  (\bibinfo{year}{2020}).
\newblock


\bibitem[\protect\citeauthoryear{Vaswani, Shazeer, Parmar, Uszkoreit, Jones,
  Gomez, Kaiser, and Polosukhin}{Vaswani et~al\mbox{.}}{2017}]%
        {Vaswani2017AttentionIA}
\bibfield{author}{\bibinfo{person}{Ashish Vaswani}, \bibinfo{person}{Noam~M.
  Shazeer}, \bibinfo{person}{Niki Parmar}, \bibinfo{person}{Jakob Uszkoreit},
  \bibinfo{person}{Llion Jones}, \bibinfo{person}{Aidan~N. Gomez},
  \bibinfo{person}{Lukasz Kaiser}, {and} \bibinfo{person}{Illia Polosukhin}.}
  \bibinfo{year}{2017}\natexlab{}.
\newblock \showarticletitle{Attention is All you Need}.
\newblock \bibinfo{journal}{\emph{ArXiv}}  \bibinfo{volume}{abs/1706.03762}
  (\bibinfo{year}{2017}).
\newblock


\bibitem[\protect\citeauthoryear{Zhang, Du, Sun, Donohue, and Dai}{Zhang
  et~al\mbox{.}}{2021}]%
        {Zhang2021Form1I}
\bibfield{author}{\bibinfo{person}{Yanci Zhang}, \bibinfo{person}{Tianming Du},
  \bibinfo{person}{Yujie Sun}, \bibinfo{person}{Lawrence Donohue}, {and}
  \bibinfo{person}{Rui Dai}.} \bibinfo{year}{2021}\natexlab{}.
\newblock \showarticletitle{Form 10-Q Itemization}.
\newblock \bibinfo{journal}{\emph{ArXiv}}  \bibinfo{volume}{abs/2104.11783}
  (\bibinfo{year}{2021}).
\newblock


\bibitem[\protect\citeauthoryear{Zhao, Zhan, and Jie}{Zhao
  et~al\mbox{.}}{2018}]%
        {ZHAO2018619}
\bibfield{author}{\bibinfo{person}{Xinyan Zhao}, \bibinfo{person}{Mengqi Zhan},
  {and} \bibinfo{person}{Cheng Jie}.} \bibinfo{year}{2018}\natexlab{}.
\newblock \showarticletitle{Examining multiplicity and dynamics of publics’
  crisis narratives with large-scale Twitter data}.
\newblock \bibinfo{journal}{\emph{Public Relations Review}}
  \bibinfo{volume}{44}, \bibinfo{number}{4} (\bibinfo{year}{2018}),
  \bibinfo{pages}{619--632}.
\newblock
\showISSN{0363-8111}
\urldef\tempurl%
\url{https://doi.org/10.1016/j.pubrev.2018.07.004}
\showDOI{\tempurl}


\bibitem[\protect\citeauthoryear{Zhu}{Zhu}{2020}]%
        {zhu2020adaptive}
\bibfield{author}{\bibinfo{person}{Liao Zhu}.} \bibinfo{year}{2020}\natexlab{}.
\newblock \emph{\bibinfo{title}{The Adaptive Multi-Factor Model and the
  Financial Market}}.
\newblock \bibinfo{thesistype}{Ph.D. Dissertation}. \bibinfo{school}{Cornell
  University}.
\newblock


\bibitem[\protect\citeauthoryear{Zhu, Basu, Jarrow, and Wells}{Zhu
  et~al\mbox{.}}{2020}]%
        {zhu2020high}
\bibfield{author}{\bibinfo{person}{Liao Zhu}, \bibinfo{person}{Sumanta Basu},
  \bibinfo{person}{Robert~A. Jarrow}, {and} \bibinfo{person}{Martin~T. Wells}.}
  \bibinfo{year}{2020}\natexlab{}.
\newblock \showarticletitle{High-Dimensional Estimation, Basis Assets, and the
  Adaptive Multi-Factor Model}.
\newblock \bibinfo{journal}{\emph{Quarterly Journal of Finance}}
  \bibinfo{volume}{10}, \bibinfo{number}{04} (\bibinfo{year}{2020}),
  \bibinfo{pages}{2050017}.
\newblock


\bibitem[\protect\citeauthoryear{Zhu, Jarrow, and Wells}{Zhu
  et~al\mbox{.}}{2021a}]%
        {zhu2021time}
\bibfield{author}{\bibinfo{person}{Liao Zhu}, \bibinfo{person}{Robert~A.
  Jarrow}, {and} \bibinfo{person}{Martin~T. Wells}.}
  \bibinfo{year}{2021}\natexlab{a}.
\newblock \showarticletitle{Time-Invariance Coefficients Tests with the
  Adaptive Multi-Factor Model}.
\newblock \bibinfo{journal}{\emph{arXiv preprint arXiv:2011.04171}}
  (\bibinfo{year}{2021}).
\newblock


\bibitem[\protect\citeauthoryear{Zhu, Sun, and Wells}{Zhu
  et~al\mbox{.}}{2021b}]%
        {zhu2021clustering}
\bibfield{author}{\bibinfo{person}{Liao Zhu}, \bibinfo{person}{Ningning Sun},
  {and} \bibinfo{person}{Martin~T. Wells}.} \bibinfo{year}{2021}\natexlab{b}.
\newblock \showarticletitle{Clustering Structure of Microstructure Measures}.
\newblock \bibinfo{journal}{\emph{arXiv preprint arXiv:2107.02283}}
  (\bibinfo{year}{2021}).
\newblock


\bibitem[\protect\citeauthoryear{Zhu, Wu, and Wells}{Zhu
  et~al\mbox{.}}{2021c}]%
        {zhu2021news}
\bibfield{author}{\bibinfo{person}{Liao Zhu}, \bibinfo{person}{Haoxuan Wu},
  {and} \bibinfo{person}{Martin~T. Wells}.} \bibinfo{year}{2021}\natexlab{c}.
\newblock \showarticletitle{A News-based Machine Learning Model for Adaptive
  Asset Pricing}.
\newblock \bibinfo{journal}{\emph{arXiv preprint arXiv:2106.07103}}
  (\bibinfo{year}{2021}).
\newblock


\end{thebibliography}

\end{document}